\documentclass{bmcart}

\usepackage{amsthm,amsmath}

\usepackage{booktabs}
\usepackage{url}
\usepackage[utf8]{inputenc} 
\usepackage{amsmath,amsthm}
\usepackage{amssymb}
\usepackage{siunitx}
\usepackage{graphicx}
\usepackage{enumitem}

\usepackage{tikz}
\usepackage{listings}

\usepackage{xcolor}
\newcommand\replies[1]{\textcolor{blue}{#1}}
\newcommand\JP[1]{\textcolor{cyan}{#1}}
\newcommand\JB[1]{\textcolor{orange}{#1}}
\newcommand{\R}{\mathbb{R}}
\newcommand{\eps}{\varepsilon}
\newcommand{\supp}{\text{supp}\,}
\def\abs#1{\left|#1\right|}
\def\norm#1{\left\|#1\right\|}
\newcommand{\coloneqq}{\mathrel{\mathop:}=}
\newcommand{\Am}{\mathcal{A}}
\newcommand{\A}{\mathcal{A}}
\newcommand{\Ll}{{L}}
\newcommand{\B}{{B}}
\newcommand{\Lls}{\mathcal{L}}
\newcommand{\Eb}{\mathbf{E}}
\newcommand{\Mb}{\mathbf{M}}
\newcommand{\Wb}{\mathbf{W}}
\newcommand{\Ab}{\mathbf{A}}
\newcommand{\Pb}{\mathbf{P}}
\newcommand{\Sb}{\mathbf{S}}
\newcommand{\Ib}{\mathbf{I}}
\newcommand{\Lb}{\mathbf{L}}
\newcommand{\Db}{\mathbf{D}}
\newcommand{\Kb}{\mathbf{K}}
\newcommand{\Ms}{\mathcal{M}}
\newcommand{\Ns}{\mathcal{N}_s}
\newcommand{\La}{L}
\newcommand{\Laa}{\mathcal{L}}
\newcommand{\M}{{M}}
\newcommand{\tM}{\widetilde{\mathcal{M}}}
\newcommand{\hM}{\widehat{\mathcal{M}}}
\newcommand{\N}{{N}}
\newcommand{\Y}{{Y}}
\newcommand{\U}{{U}}
\newcommand{\cc}{{c}_0}
\newcommand{\V}{\mathcal V}
\newcommand{\Pp}{{P}}
\newcommand{\E}{\mathcal E}
\newcommand{\D}{\mathcal D}
\newcommand{\K}{\mathcal K}
\newcommand{\npinns}{{\sc nPINNs} \ }
\renewcommand{\d}{\,\mathrm d}
\newcommand{\dy}{\mathrm d y}
\newcommand{\dx}{\mathrm d x}
\newcommand{\ddx}{\frac{\mathrm d}{\dx}}
\newcommand{\by}{\boldsymbol{y}}
\newcommand{\bp}{\boldsymbol{p}}
\newcommand{\bv}{\boldsymbol{v}}
\renewcommand{\bf}{\boldsymbol{f}}
\newcommand{\bu}{\boldsymbol{u}}

\newtheorem{theorem}{Theorem}

\newcommand{\Xdata}{X_{\text{data}}}
\newcommand{\pinn}{\textsc{PINN}}
\newcommand{\pinns}{\textsc{PINNs}}
\newcommand{\adam}{\textsc{ADAM}}
\newcommand{\lbfgs}{\textsc{L-BFGS}}
\newcommand{\fvm}{\textsc{FVM}}
\newcommand{\graphPINN}{\textsc{graphPINN}}
\newcommand{\edgePINN}{\textsc{edgePINN}}

\newcommand{\martin}[1]{\textcolor{violet}{[\textbf{Martin:} #1]}}
\newcommand{\janP}[1]{\textcolor{red}{[\textbf{JanP:} #1]}}
\newcommand{\janB}[1]{\textcolor{blue}{[\textbf{JanB:} #1]}}
\newcommand{\maxW}[1]{\textcolor{orange}{[\textbf{Max:} #1]}}
\newcommand{\tomR}[1]{\textcolor{olive}{[\textbf{Tom:} #1]}}

\newtheorem{thm}{Theorem}[section]

\newcommand*{\email}[1]{\href{mailto:#1}{\detokenize{#1}}}
\definecolor{orcidlogocolor}{HTML}{A6CE39}
\newcommand*{\orcid}[1]{\href{https://orcid.org/#1}{ORCID~\detokenize{#1}}}

\RequirePackage{caption}
\RequirePackage{subcaption}

\startlocaldefs
\endlocaldefs

\begin{document}

	\begin{frontmatter}

\begin{fmbox}
	\dochead{Research}
	
	
	\title{A comparison of PINN approaches for drift-diffusion equations on metric graphs}
	
	
	\author[
	addressref={aff1},                   
	email={jan.blechschmidt@math.tu-chemnitz.de}   
	]{\inits{T}\fnm{Jan} \snm{Blechschmidt}}
	\author[
	addressref={aff1},                   
	email={jfpietschmann@math.tu-chemnitz.de}   
	]{\inits{T}\fnm{Jan-F.} \snm{Pietschmann}}
	\author[
	addressref={aff1},                   
	email={tom-christian.riemer@math.tu-chemnitz.de}   
	]{\inits{T}\fnm{Tom-C.} \snm{Riemer}}
	\author[
	addressref={aff1},                   
	corref={aff1},                       
	email={martin.stoll@math.tu-chemnitz.de}   
	]{\inits{T}\fnm{Martin} \snm{Stoll}}
	\author[
	addressref={aff1},                   
	email={max.winkler@math.tu-chemnitz.de}   
	]{\inits{T}\fnm{Max} \snm{Winkler}}

	
	\address[id=aff1]{
		\orgdiv{Department of Mathematics},             
		\orgname{TU Chemnitz},          
		\city{Chemnitz},                              
		\cny{Germany}                                    
	}
	
	
\end{fmbox}


\begin{abstractbox}
	
	\begin{abstract} 
 In this paper we focus on comparing machine learning approaches for quantum graphs, which are metric graphs, i.e., graphs with dedicated edge lengths, and an associated differential operator. In our case the differential equation is a drift-diffusion model. Computational methods for quantum graphs require a careful discretization of the differential operator that also incorporates the node conditions, in our case Kirchhoff-Neumann conditions. Traditional numerical schemes are rather mature but have to be tailored manually when the differential equation becomes the constraint in an optimization problem. Recently, physics informed neural networks (\pinns) have emerged as a versatile tool for the solution of partial differential equations from a range of applications. They offer flexibility to solve parameter identification or optimization problems by only slightly changing the problem formulation used for the forward simulation. We compare several \pinn\ approaches for solving the drift-diffusion on the metric graph.   
	\end{abstract}
	
	
	
	
\end{abstractbox}
%

	\end{frontmatter}
	
	
	
\section{Introduction}
Dynamic processes on networks (graphs) \cite{newman2018networks,barabasi2013network} are crucial for understanding complex phenomena in many application areas. 
Here we focus on metric graph where each edge can be associated with an interval thus allowing the definition of differential operators (quantum graphs) \cite{lagnese2012modeling,berkolaiko2013introduction}.
Numerical methods for such graphs have gained recent interest \cite{arioli2018finite, gyrya2019explicit, stoll2021optimal} both for simulation and optimal control. As the structure of the network can possibly become rather complex efficient schemes 
such as domain decomposition methods \cite{leugering2017domain} 
are often needed.

In this paper we investigate the use of the recently introduced physics-informed neural networks (\pinns). In their seminal paper \cite{raissi2019physics} Raissi and co-authors introduce the idea of utilizing modern machine learning ideas as well as the computational frameworks such as TensorFlow \cite{tensorflow2015-whitepaper} to solve differential equations. Many authors have extended the applicability of the \pinn\ approach to different application areas  \cite{zhu2019physics,mao2020physics,jin2021nsfnets,sahli2020physics}
including fluid dynamics \cite{raissi2018hidden,MAO2020112789,lye2020deep,magiera2020constraint,wessels2020neural}, continuum mechanics and elastodynamics \cite{haghighat2020deep,nguyen2020deep,rao2020physics},
inverse problems \cite{meng2020composite,jagtap2020conservative},
fractional advection-diffusion equations
\cite{pang2019fpinns},
stochastic advection-diffusion-reaction equations
\cite{chen2019learning},
stochastic differential equations
\cite{yang2020physics} and
power systems
\cite{misyris2020physics}.
\textsc{XPINN}s (eXtended \pinns) are introduced in~\cite{JagtapKardiadakis2020} as a generalization of \pinns\ involving multiple neural networks allowing for parallelization in space and time via domain decomposition, see also \cite{Heinlein2021} for a recent review on machine learning approaches in domain decomposition. 

Our goal is to illustrate that the \pinn\ framework can also be applied to the case of an drift-diffusion equation posed on a metric graph. Those models are used in many application areas ranging from semiconductor modeling \cite{natalini1996bipolar}, solar cells \cite{hwang2009drift} to modeling electrical networks, cf. \cite{hinze2011pod}, and thus serve as a relevant and sufficiently complex test case.

Classical approaches
are Galerkin finite element methods \cite{brenner2008mathematical,Egger2020} to discretize of the governing equations. In particular, discontinuous Galerkin finite element method (DG) \cite{cockburn2012discontinuous} or finite volume methods \cite{eymard2000finite} are often the method of choice given they are guaranteeing local conservation of the flows. 
We use such a finite volume scheme to obtain a reference solution.


We implement and compare several \pinn-based machine learning approaches. These become very attractive numerical schemes for PDEs as they are very flexible when tasked with solving parameter identification tasks. Something that is notoriously difficult in the context of finite element methods as for these schemes a careful discretization of the optimality conditions often requires rather technical derivations as well as the solution of complex systems of linear equations associated with the first order conditions or linearizations thereof \cite{hinze2008optimization}.

We start with a brief introduction to the drift-diffusion models in Section \ref{sec:DD_on_graphs}. In Section \ref{sec::pinns} we introduce the several different \pinn\ setups by detailing their architectures and the corresponding loss functions that then reflect all the intricate relations within the network. In Section \ref{sec::fvm} we briefly review a finite volume discretization of the drift diffusion to which we compare the \pinn-based methods. This and other comparisons are given in Section \ref{sec::results} where we focus on a number of \pinn\ approaches and their numerical realization.
\section{Drift-diffusion equations on metric graphs}\label{sec:DD_on_graphs}

Let us introduce our notion of a metric graph in more detail. A metric graph is a directed graph that consists of a set of vertices $\V$ and edges $\E$ connecting a pair of vertices denoted by $(v_e^o,v_e^t)$ where $v_e^o,v_e^t\in \V$. Here $v_e^o$ stands or the vertex at the origin while $v_e^t$ denotes the terminal vertex. In contrast to combinatorial graphs a length $\ell_e$ is assigned to each edge $e\in \E.$ Thus identifying each edge with a one-dimensional interval allows for the definition of differential operators. We also introduce a normal vector $n_e(v)$ defined as $n_e(v_e^o) = -1$ and $n_e(v_e^t) = 1$. 
To prescribe the behavior at the boundary of the graph, we first subdivide the set of vertices $\V$ into the interior vertices $\mathcal{V}_\mathcal{K}$ and the exterior vertices $\mathcal{V}_\mathcal{D}$ as follows
\begin{itemize} 
	\item the set of interior vertices $v \in \mathcal{V}_\mathcal{K} \subset \mathcal{V}$, contains all vertices that are incident to at least one incoming edge and at least one outgoing edge (i.e. $\forall v \in \mathcal{V}_\mathcal{K} \; \exists \ e_1, e_2 \in \mathcal{E}$ such that $v^{\operatorname{t}}_{e_1} = v$ and $v^{\operatorname{o}}_{e_2} = v$) 

	\item the set of exterior vertices $v \in \mathcal{V}_\mathcal{D} \coloneqq \mathcal{V} \setminus \mathcal{V}_\mathcal{K}$, contains vertices to which either only incoming or only outgoing edges are incident (i.e. either $v^{\operatorname{t}}_{e} = v$ or $v^{\operatorname{o}}_{e} = v$ holds $\forall e \in \mathcal{E}_v$)
	
\end{itemize}

One possible application of such a setting could be a road network as depicted in Figure~\ref{fig8}.
\begin{figure}
    \begin{center}
        \begin{subfigure}[b]{0.3\textwidth}
            \begin{center}
                \includegraphics[scale=0.30]{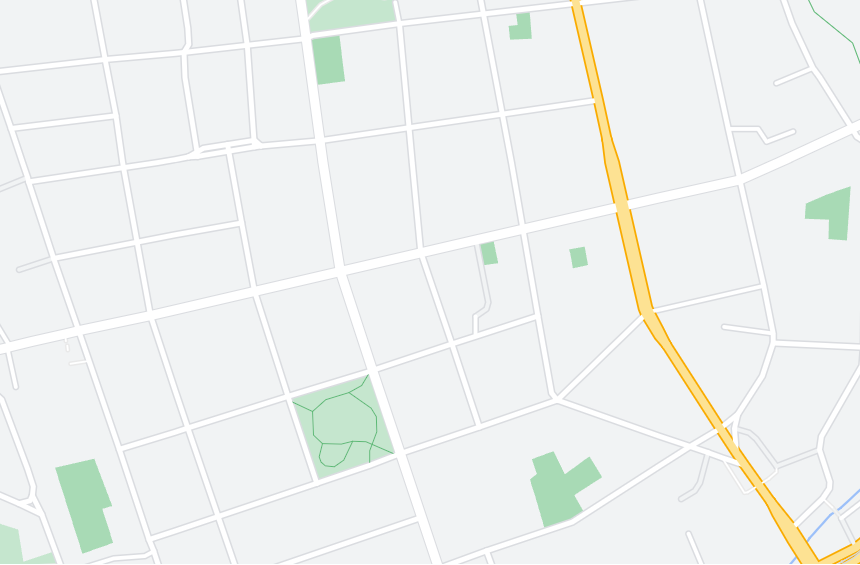}
            \end{center}
            \caption{A street network in Chemnitz, Saxony, Germany. Image from Google maps. Coordinates of the central point: $50.83$, $12.90$.}
            \label{fig8:f1}
        \end{subfigure}\hspace{15mm}
        \begin{subfigure}[b]{0.5\textwidth}
            \begin{center}
                \includegraphics[scale=0.16]{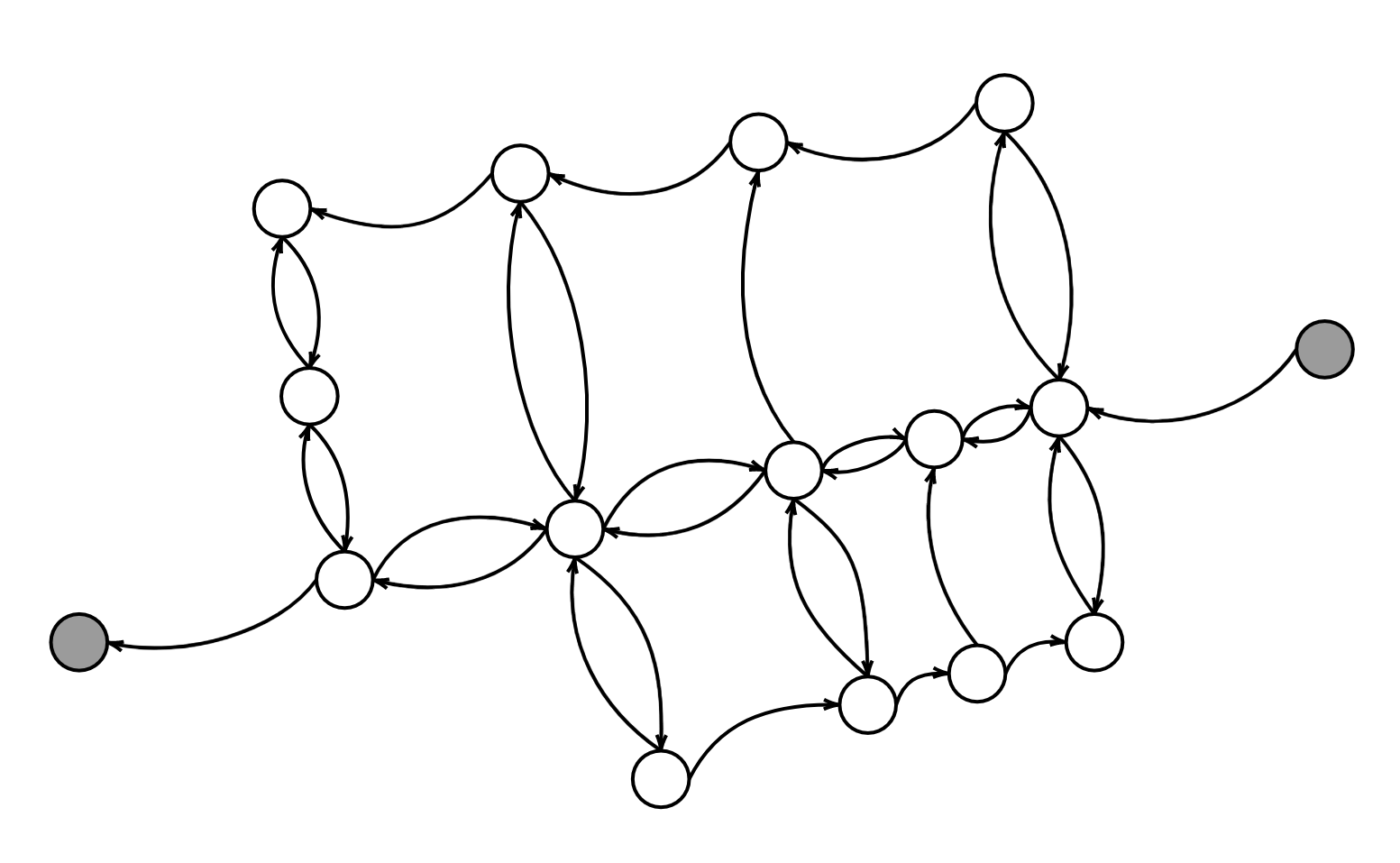}
            \end{center}
            \caption{A metric graph modeling a compact road network within the left road network. Empty circles are interior vertices while filled ones depict exterior ones.}
            \label{fig8:f2}
        \end{subfigure}
    \end{center}
    \caption{A road network map and a modeling metric graph.}
    \label{fig8}
\end{figure}

On a metric graph, it makes sense to consider differential operators defined on each edge, and we focus on non-linear drift-diffusion equations
\begin{equation}
    \label{eq:strong_pde}
    \partial_t\rho_e  = \partial_x (\eps\,\partial_x\rho_e - f(\rho_e)\,\partial_x V_e), \quad e \in \E,
\end{equation}
where $\rho_e : e \times (0,T) \to \R_+$ describes, on each edge, the concentration of some quantity while $V_e: e \times (0,T) \to \R_+$ is a given potential and $\eps > 0$ a (typically small) diffusion constant. Furthermore, $f: \R_+ \to \R_+$ satisfies $f(0) = f(1) = 0$. This property ensures that solutions satisfy $0 \le \rho_e \le 1$ a.e. on each edge, see Theorem \ref{thm:existence}. By identifying each edge with an interval $[0,\ell_e]$, we define the flux as
\begin{align} \label{eq:flux}
    J_e(x) := - \eps\,\partial_x \rho_e (x) + f(\rho_e(x))\,\partial_x V_e(x).
\end{align}
A typical choice for $f$ used in the following is $f(\rho_e) = \rho_e(1-\rho_e)$.

To make \eqref{eq:strong_pde} a well-posed problem, we need to add initial-conditions as well as coupling conditions on the vertices. First we impose on each edge $e \in \mathcal{E}$ the following initial condition
\begin{equation}
    \label{eq:initial_conditions}
    \rho_e \left( 0,x \right)  = \rho_{e, 0} \left( x \right) ,
\end{equation}
where $\rho_{e, 0} \in L^2 \left( e \right) $ returns the given density on each point of the edge $e$ at the start time of the observation $t=0$.

For vertices $v\in \V_\K\subset \V$, we apply \emph{homogeneous Neumann--Kirchhoff conditions}, i.\,e., there holds
\begin{equation}
    \label{eq:Kirchhoff_Neumann_condition}
    \sum_{e\in \E_v}J_e(v)\,n_e (v)=0
\end{equation}
with $\E_v$ the edge set incident to the vertex $v$.
Additionally, we ask the solution to be continuous over vertices, i.e.
\begin{align}
    \label{eq:continuity_condition}
    \rho_e(v) = \rho_{e'}(v) \quad \text{ for all }v \in \V_\K,\; e,\,e' \in \E_v.
\end{align}
In vertices $v\in \V_\D:=\V\setminus \V_\K$ the solution $\rho$ fulfills \emph{flux boundary conditions}
\begin{equation}
    \label{eq:Dirichlet_conditions}
    \sum_{e\in \E_v}J_e(v) n_e (v)=-\alpha_v(t) (1-\rho_v) + \beta_v(t) \rho_v,
\end{equation}
where
$\alpha_v:(0,T) \to \R_+, \, \beta_v : (0,T) \to \R_+$, ${v \in \V_\D}$,
are functions prescribing the rate of influx of mass into the network ($\alpha_v$) as well as the velocity of mass leaving the network ($\beta_v$) at the boundary vertices. Note that this choice ensures that the bounds $0 \le \rho_e \le 1$ are preserved. In typical situations, boundary vertices are either of influx- or of outflux type, i.e. $\alpha_v \beta_v = 0$ for all $v \in \V_\D$.

The Kirchhoff-Neumann conditions are the natural boundary conditions for the differential operator \eqref{eq:strong_pde}, as they ensure that mass enters or leaves the system only via the boundary nodes $\V_\D$ for which either $\alpha_v$ or $\beta_v$ is positive.

Having introduced the complete continuous model, we state the following existence and uniqueness result, which can be obtained by combining the proofs of \cite{Burger2020,Egger2020}.
\begin{thm}\label{thm:existence}
Let the initial data $\rho_0 \in L^2(\Gamma)$ satisfy $0 \le \rho_0 \le 1$
a.e.\ on $\E$ and let nonnegative functions $\alpha_v, \beta_v \in L^\infty(0,T)$, $v\in \V_D$ be given. Then there exists a unique weak solution $\rho \in L^2(0,T; H^1(\Gamma)) \cap H^1(0,T; H^1(\Gamma)^*)$ s.t.
    \begin{align*}
        &\sum_{e \in \E} \left(\int_e  \partial_t \rho_e\,\varphi_e \;dx +
        \int_e (\varepsilon\,\partial_x \rho_e -
        f(\rho_e)\,\partial_x V_e)\,\partial_x \varphi_e \;dx\right)\\
        &\qquad + \sum_{v \in \V_D} (-\alpha_v(t) (1-\rho_v) + \beta_v(t) \rho_v)\varphi(v) = 0,
    \end{align*}
    for all test functions $\varphi \in H^1(\Gamma)$. Here $L^2$ denotes the space of square integrable functions. The space $H^1$ denotes the space of functions for which also the weak derivative is bounded in $L^2$ and with $(H^1)^*$ its dual space. The Bochner spaces contain time-dependent functions where for $u=(x,t)$ to belong to, e.g. $L^2(0,T; H^1(\Gamma))$, the norm
    $$
    \int_0^T \|u(\cdot, t)\|_{H^1(\Gamma)}^2\;dx 
    $$
    has to be finite.
\end{thm}
	
    We note that through \eqref{eq:strong_pde} we obtain the following differential operator defined on the metric graph $\Gamma$
	\begin{equation} 
		\label{eq:Hamiltonian}
		\mathcal{H} [\rho_e]  \left( t,x \right)  \coloneqq \partial_t \rho_e  \left( t,x \right)   - \partial_x  \left( \varepsilon \partial_x \rho_e  \left( t,x \right)  - f \left( \rho_e  \left( t,x \right)   \right)  \partial_x V_e  \left( t,x \right)  \right) ,
	\end{equation}
	and together with the above mentioned initial and vertex conditions we have obtained the quantum graph that we want to tackle with our \pinn\ approaches.
\section{Physics-informed neural networks on graphs} \label{sec::pinns}

In recent decades, machine learning methods using deep neural networks proved to be a versatile tool for approximating arbitrary functions.
This motivated their use for the approximation of solutions to (partial) differential equations \cite{raissi2019physics}.
The classical approach of supervised learning trains a neural network purely using labeled data, only.
Two challenges may arise in this situation. First, when analyzing complex physical, biological or technical systems, the cost of data acquisition is often prohibitively expensive and one is faced with the challenge of drawing conclusions and making decisions based on incomplete information. For the resulting small amount of data, most of the modern machine learning methods are not robust enough and offer no convergence guarantees. Second, neural networks do not take the physical properties underlying the system into account. Without this prior information, the solution is often not unique and may lose physical correctness. Adding this prior information as an additional constraint to the learning problem restricts the space of admissible solutions.

\pinns{} are universal function approximators, which are trained using unlabeled data, but in addition utilize the governing physical equation as an additional constraint in the learning phase. In turn, encoding such structured information in a learning algorithm leads to an increase in the information content of the data. This ensures that the learning quickly moves towards the solution and then generalizes well.

Moreover, \pinns\ can be used for data-driven discovery of partial differential equations or system identification, simply by adding a data fidelity term to the loss functional (cf. \cite[pp.~1-2]{RaissiPerdikarisKarniadakis2017}).
However, we concentrate on the problem of computing data-driven approximate solutions to the non-linear PDE system given by \eqref{eq:strong_pde} under the given initial and vertex conditions.


In the following, we present three different approaches to use \pinns\ to approximate the solution of \eqref{eq:strong_pde} under the initial and vertex conditions. The first two correspond to the continuous time model approach and the last to the discrete time model approach, both presented in \cite{RaissiPerdikarisKarniadakis2017}. As the domain in the metric graph case is more structured than a standard two- or three-dimensional domain, we exploit this given structure for constructing these three \pinn\ approaches.

\subsection{Two continuous time models}

The idea of the continuous time model \pinn\ approach is based on constructing a residual function from \eqref{eq:strong_pde} for each edge $e \in \mathcal{E}$ of the graph, i.e. 
\begin{equation}
	\label{eq:residual_function}
	r_{e} \left( t,x \right)=\partial_t \rho_{e} \left( t,x \right) - \partial_x   \left(  \varepsilon \partial_x  \rho_{e} \left( t,x \right) - f \left( \rho_{e} \left( t,x \right) \right) \partial_x V \left( t,x \right) \right),
\end{equation}
and on approximating the unknown solution $\rho_e$ of \eqref{eq:strong_pde} on each edge $e \in \mathcal{E}$ by a neural network $\rho_{\theta_e} \colon \mathbb{R}^2 \to \mathbb{R}$. Here, $\theta_e$ denotes the trainable parameters of the solution $\rho_e$ on an individual edge $e \in \mathcal{E}$. 

Substituting the neural network $\rho_{\theta_e}$ into \eqref{eq:residual_function} results in the definition of a so-called residual network for each edge $e \in \mathcal{E}$, which is given by 
\begin{equation}
	\label{eq:residual_network}
	r_{\theta_e} \left( t,x \right)=\partial_t \rho_{\theta_e} \left( t,x \right) - \partial_x   \left(  \varepsilon \partial_x  \rho_{\theta_e} \left( t,x \right) - f \left( \rho_{\theta_e} \left( t,x \right) \right) \partial_x V \left( t,x \right) \right).
\end{equation}
To obtain the residual network $r_{\theta_e}$ from $\rho_{\theta_e}$, one exploits the techniques of automatic differentiation to differentiate $\rho_{\theta_e}$ with respect to its input values $t$ and $x$. We note that this residual network $r_{\theta_e} \colon \mathbb{R}^2 \to \mathbb{R}$ is in fact also a neural network and has the same set of trainable parameters $\theta_e$ as the approximating network $\rho_{\theta_e}$. In the learning phase, the network $\rho_{\theta_e}$ is trained to satisfy \eqref{eq:strong_pde} by minimizing the norm of the residual network $r_{\theta_e}$ with respect to the trainable parameters $\theta_e$ over a set of collocation points.

In order to train each neural network $\rho_{\theta_e}$ to fulfill the initial, coupling and boundary conditions \eqref{eq:initial_conditions}--\eqref{eq:Dirichlet_conditions}, further misfit terms must be constructed and incorporated into the cost function. There are several possibilities to define these terms as well as how to combine them to one or more cost functions.

We present two continuous time model \pinn\ approaches, which differ in their methodology for training the neural networks $\rho_{\theta_e}$. On the one hand, this difference results from the definition of the cost functions, into which the residual network and the relevant misfit terms are incorporated, and on the other hand, they differ in the resulting consideration of the trainable parameters to be minimized and the sequential process in the learning phase, simultaneously or alternatingly. We give a detailed description of each approach in the following.

	\subsection{\graphPINN\ approach}
	\label{subsec:graphPINN}

	For this approach, we construct one single cost function approximating the problem on all edges simultaneously, including initial, boundary and coupling conditions.
	Treating the complete problem with one single cost function, rather than optimizing each network on each edge separately, motivates the name \graphPINN.
	This single cost function consists of 
	the following misfit terms:

	\begin{itemize}
		\item \emph{Residual misfit term}: using \eqref{eq:residual_network} we enforce the structured information imposed by \eqref{eq:strong_pde} for each individual edge $e \in \mathcal{E}$ via \begin{equation} \label{eq:misfit_residual} \phi_{e,r}  \left( X_e \right) \coloneqq \frac{1}{n_e} \sum_{i=1}^{n_e} r_{\theta_e}  \left( t_e^i, x_e^i  \right)^2, \end{equation} where $X_e = \left\{ \left( t_e^i, x_e^i \right) \right\}_{i=1}^{n_e} \subset \left( 0, T \right) \times \left[0, \ell_e\right]$ is a set of time-space collocation points that are drawn randomly or chosen equidistantly. 
		\item \emph{Kirchhoff-Neumann misfit term}: to enforce the Kirchhoff-Neumann conditions, \eqref{eq:Kirchhoff_Neumann_condition}, we define the following misfit term for each interior vertex $v \in \mathcal{V}_\mathcal{K}$ \begin{equation} \label{eq:misfit_Kirchhoff} \phi_{v,K}  \left( X_{v,b} \right) \coloneqq \frac{1}{n_b} \sum_{i=1}^{n_b}  \left( \sum_{e \in \mathcal{E}_v}  J_{\theta_e}\left( t_{v,b}^i, v \right)  n_e  \left( v \right) \right)^2,  \end{equation} with \begin{equation} \label{eq:misfit_flux} J_{\theta_e}\left( t_{v,b}^i, v \right) = \left( - \varepsilon \partial_x \rho_{\theta_e}  \left( t_{v,b}^i, v \right) + f \left( \rho_{\theta_e}  \left( t_{v,b}^i, v \right) \right) \partial_x V_e \left( t_{v,b}^i, v \right) \right), \end{equation} where $X_{v,b} = \left\{ t_{v,b}^i \right\}_{i=1}^{n_b} \subset \left( 0,T \right)$ is a set of time snapshots where the Kirchhoff-Neumann conditions are enforced. The values $\left\{ \rho_{\theta_e}  \left( t_{v,b}^i, v \right) \right\}_{i=1}^{n_b}$ are either equal to $\left\{ \rho_{\theta_e}  \left( t_{v,b}^i, 0 \right) \right\}_{i=1}^{n_b}$ if the interior vertex $v \in \mathcal{V}_\mathcal{K}$ is an origin vertex of the edge $e$ (i.e. $\operatorname{o}(e) = v$), or equal to $\left\{ \rho_{\theta_e}  \left( t_{v,b}^i, \ell_e \right) \right\}_{i=1}^{n_b}$ if $v$ is a terminal vertex of the edge $e$ (i.e. $\operatorname{t}(e) = v$). Of course this also applies to the values $\left\{ \partial_x \rho_{\theta_e}  \left( t_{v,b}^i, v \right) \right\}_{i=1}^{n_b}$. We note that the derivatives are always taken into the outgoing direction.
		\item \emph{Continuity misfit term}: in order to enforce the continuity of $\left\{\rho_{\theta_e} \right\}_{e \in \mathcal{E}}$ in the interior vertices $\mathcal{V}_\mathcal{K}$, as required by \eqref{eq:continuity_condition}, we introduce two different misfit terms to accomplish this. Both misfit terms enforce by minimization that the neural networks $\left\{\rho_{\theta_e} \right\}_{e \in \mathcal{E}_v}$, which approximate the solution on the edges incident to this vertex $v \in \mathcal{V}_\mathcal{K}$, to have the same value at this vertex $v$, and can therefore be considered equivalent. In the following numerical experiments, we draw attention to which of these two misfit terms is used. The first misfit term is defined for each interior vertex $v \in \mathcal{V}_\mathcal{K}$ by \begin{equation} \label{eq:misfit_continuity} \phi_{v,c}  \left( X_{v,b} \right) \coloneqq \frac{1}{n_b} \sum_{e \in \mathcal{E}_v} \sum_{i=1}^{n_b} \left(  \rho_{\theta_e}  \left( t_{v,b}^i, v \right) - \rho_{v}^i \right)^2, \end{equation} with $X_{v,b} = \{t_{v,b}^i\}_{i=1}^{n_b}$ as given before. Here, we introduce for each interior vertex $v \in \mathcal{V}_\mathcal{K}$ some additional trainable parameters $\left\{ \rho_{v}^i \right\}_{i=1}^{n_b}$, which must be taken into account when minimizing the resulting cost function in the case of use. The second misfit is defined for each interior vertex $v \in \mathcal{V}_\mathcal{K}$ by \begin{equation} \label{eq:misfit_continuity_average} \phi_{v,c}  \left( X_{v,b} \right) \coloneqq \frac{1}{n_b}  \sum_{i=1}^{n_b} \left( \sum_{e \in \mathcal{E}_v} \left( \rho_{\theta_e}  \left( t_{v,b}^i, v \right) - \frac{1}{\abs{\mathcal{E}_v}} \sum_{e \in \mathcal{E}_v} \rho_{\theta_e}  \left( t_{v,b}^i, v \right) \right) \right)^2. \end{equation} Here, on average over all time collocation points $\{t_{v,b}^i\}_{i=1}^{n_b}$, the value $\rho_{\theta_e}  \left( t_{v,b}^i, v \right)$ of each edge $e \in \mathcal{E}_v$ should be equal to the average of all edges connected to this interior vertex $v \in \mathcal{V}_\mathcal{K}$. Both misfit terms have their numerical advantages and disadvantages.
                The misfit term defined by \eqref{eq:misfit_continuity} is simpler and requires less effort to obtain the derivatives with respect to the learnable parameters in the optimization method.
                However, it comes at a cost since additional learnable parameters are needed for each inner vertex $v \in \V_k$.
                In contrast, no further learnable parameters need to be defined for the misfit term given by \eqref{eq:misfit_continuity_average}, but it is more complex for that. 
		\item \emph{Flux misfit term}: we enforce the flux vertex conditions given by \eqref{eq:Dirichlet_conditions} for each exterior vertex $v \in \mathcal{V}_\mathcal{D}$ by defining the following misfit term \begin{equation} \label{eq:misfit_Dirichlet} \begin{aligned} \phi_{v,D}  \left( X_{v,b} \right) \coloneqq & \frac{1}{n_b} \sum_{i=1}^{n_b} \bigg( \sum_{e \in \mathcal{E}_v} J_{\theta_e}\left( t_{v,b}^i, v \right) n_e  \left( v \right) + \\
        & \alpha_v \left( t_{v,b}^i \right)  \left( 1- \rho_{\theta_e}  \left( t_{v,b}^i, v \right) \right) - \beta_v \left( t_{v,b}^i \right) \rho_{\theta_e}  \left( t_{v,b}^i, v \right) \bigg)^2, \end{aligned} \end{equation} with $X_{v,b} = \{t_{v,b}^i\}_{i=1}^{n_b}$ as introduced before and $J_{\theta_e}\left( t_{v,b}^i, v \right)$ as defined in \eqref{eq:misfit_flux}.
		\item \emph{Initial misfit term}: to enforce the initial conditions for each edge $e \in \mathcal{E}$, required by \eqref{eq:initial_conditions}, we define the following misfit term \begin{equation} \label{eq:misfit_initial} \phi_{e,0}  \left( X_{e,0} \right) \coloneqq \frac{1}{n_0} \sum_{i=1}^{n_0}  \left( \rho_{\theta_e}  \left( 0,x_{e,0}^i \right) - \rho_{e,0} \left( x_{e,0}^i \right) \right)^2, \end{equation} where $X_{e,0} = \{x_{e,0}^i\}_{i=1}^{n_0} \subset \left[0, \ell_e\right]$ is a set of collocation points along $t=0$.
	\end{itemize}
	We combine all misfit terms defined above to form the cost function:
	
	\begin{equation}
		\label{eq:loss:1}
		\begin{aligned} 
			\Phi_{\theta} \left( \operatorname{X} \right) & = \sum_{v \in \mathcal{V}_\mathcal{D}} \phi_{v,D} \left( X_{v,b} \right) + \sum_{v \in \mathcal{V}_\mathcal{K}}  \left(  \phi_{v,K}  \left( X_{v,b} \right) + \phi_{v,c} \left( X_{v,b} \right)  \right) + \\
			& \quad + \sum_{e \in \mathcal{E}}  \left(  \phi_{e,r}  \left( X_{e,r} \right) + \phi_{e,0}  \left( X_{e,0} \right)  \right), 
		\end{aligned}
	\end{equation}
	where $\operatorname{X}$ represents the union of the different collocation points $X_e$, $X_{v,b}$ and $X_{e,0}$. 
	The index $\theta$ of $\Phi_{\theta}$ refers to the set of all trainable parameters, which is the union of the trainable parameters of all networks $\left\{\rho_{\theta_e} \right\}_{e \in \mathcal{E}}$, i.e. $\theta = \bigcup_{e \in \mathcal{E}} \ \theta_e$, and if \eqref{eq:misfit_continuity} is used as continuity misfit term $\phi_{v,c} \left( X_{v,b} \right)$, the trainable parameters $\left\{ \rho_{v}^i \right\}_{i=1}^{n_b}$ for each interior vertex $v \in \mathcal{V}_\mathcal{K}$ are appended to $\theta$. 

	\subsection{\edgePINN\ approach}
	\label{subsec:edgePINN}

	For this approach, we construct one cost function for each edge $e \in \mathcal{E}$ of the considered metric graph $\Gamma = \left( \mathcal{V}, \mathcal{E} \right)$, separately. In the cost function for each edge $e \in \mathcal{E}$, the corresponding residual and initial misfit terms are incorporated. Additional misfit terms enforce the vertex conditions at the two connected vertices, which depend on the networks of the adjacent edges. This single cost function is then minimized with respect to $\theta_e$. The name \edgePINN\ was chosen because the approximation problem is in a sense first decomposed and then considered on each edge individually. The idea for this approach is adopted from \cite{jagtap2020conservative}, in which so-called conservative physics-informed neural networks, abbreviated \textsc{cPINNs}, on discrete domains for non-linear conservation laws were presented. In \cite{jagtap2020conservative} the domain on which the relevant conservation law is defined is split into several adjacent subdomains, for each subdomain a cost function is defined, which consists of several misfit terms, and this cost function is used to train a neural network, which approximates the solution of the conservation law in this subdomain. These several misfit terms are the residual and initial misfit terms for the respective subdomain and further misfit terms, enforcing the interface conditions, i.e. flux continuity conditions in strong form as well as enforcing the network of the respective subdomain along the interface to the neighboring subdomain to be equal to the average given by this neural network and the neighboring neural network along this common interface. These resulting cost functions are then minimized several times in alternating succession on each subdomain. 
	The idea of dividing the domain into several subdomains and defining a cost function for each subdomain can easily be transferred to the approximation problem considered in this paper due to the structure given by the graph. In our case, the subdomains are the individual edges $e \in \mathcal{E}$ of the graph $\Gamma = \left( \mathcal{V}, \mathcal{E} \right)$, and a cost function has to be constructed for each individual edge $e \in \mathcal{E}$. As already mentioned, each of these edge-wise cost functions contains the residual misfit term, given by \eqref{eq:misfit_residual}, which enforces the approximation $\rho_{\theta_e}$ to satisfy \eqref{eq:strong_pde} and an initial misfit term, given by \eqref{eq:misfit_initial}, which enforces $\rho_{\theta_e}$ to satisfy \eqref{eq:initial_conditions}. Since each directed edge $e = (v^{\operatorname{o}}_e, v^{\operatorname{t}}_e) \in \mathcal{E}$ is uniquely defined by its origin vertex $v^{\operatorname{o}}_e \in \mathcal{V}$ and its terminal vertex $v^{\operatorname{t}}_e \in \mathcal{V}$, these two vertices specify the vertex conditions, which must be enforced via $\rho_{\theta_e}$ in its corresponding cost function. 
	We have to verify for each of the two vertices $v^{\operatorname{o}}_e, v^{\operatorname{t}}_e \in \mathcal{V}$ whether they are interior or exterior vertices, as the vertex conditions and the resulting necessary misfit terms are different. We define for each vertex $v \in \mathcal{V}$ the following term, which assigns the correct misfit term(s) as: 
	\begin{equation}
		\label{eq:vertex_assignment}
		\phi_{v}(X_{v,b}) = \begin{cases} \phi_{v,K}  \left( X_{v,b} \right) +  \phi_{v,c}  \left( X_{v,b} \right)& \text{if } v \in \mathcal{V}_{\mathcal{K}}, \\ \phi_{v,D}  \left( X_{v,b} \right) & \text{if } v \in \mathcal{V}_{\mathcal{D}}. \end{cases}
	\end{equation}
    Here it is convenient that the vertex conditions are evaluated at the same time snapshots $X_{v,b} = \left\{ t_{v,b}^i \right\}_{i=1}^{n_b} \subset \left( 0,T \right)$. We note, adapted from \cite{jagtap2020conservative}, that for this approach we use \eqref{eq:misfit_continuity_average} in our numerical experiments for the continuity misfit term $\phi_{v,c}  \left( X_{v,b} \right)$, which enforces the network $\rho_{\theta_e}$ at an interior vertex $v \in \mathcal{V}_{\mathcal{K}}$ to be equal to the average of the values generated by all networks of the edges in $\mathcal{E}_v$ evaluated at that vertex $v$. 
    By introducing \eqref{eq:vertex_assignment} we are now able to define the cost function for a single edge $e = (v^{\operatorname{o}}_e, v^{\operatorname{t}}_e) \in \mathcal{E}$ of the graph: 
	\begin{equation}
		\label{eq:cost:2}
		\phi_{\theta_e} \left( X \right) \coloneqq \phi_{e,r}  \left( X_e \right) + \phi_{e,0}  \left( X_{e,0} \right) + \phi_{v^{\operatorname{o}}_e}(X_{v,b}) + \phi_{v^{\operatorname{t}}_e}(X_{v,b}).
	\end{equation}
	Here, $\operatorname{X}$ represents the union of the different collocation points. It is easy to see that the misfit terms are just grouped in a different way compared to \eqref{eq:loss:1}. 
	For each individual edge $e \in \mathcal{E}$ of the graph $\Gamma = \left( \mathcal{V}, \mathcal{E} \right)$, a cost function $\phi_{\theta_e} \left( X \right)$ of the form \eqref{eq:cost:2} is constructed. In the learning phase, these cost functions $\left \{ \phi_{\theta_e} \left( X \right) \right\}_{e \in \mathcal{E}}$ are minimized several times in an alternating fashion, each $\phi_{\theta_e} \left( X \right)$ with respect to the trainable parameters $\theta_e$ of the respective edge $e$. 
	We note that for a connected graph (a graph where each vertex is incident to at least one edge) with more than two vertices, there are three different cases on an edge $e = (v^{\operatorname{o}}_e, v^{\operatorname{t}}_e) \in \mathcal{E}$. Either the origin and terminal vertices are both interior vertices or only one of the two is an exterior vertex. An edge with two exterior vertices would be a closed system in our view, since mass flows into the system via the origin vertex and immediately flows out of the system via the terminal vertex.
	We point out that, if the graph consists of more than one edge, the cost function $\phi_{\theta_e} \left( X \right)$ of a single edge $e \in \mathcal{E}$ always includes the trainable parameters of the neural networks that approximate the solution on the edges that are incident to the origin and terminal vertices of the relevant edge $e$ (via the misfit terms that enforce the vertex conditions, since in those we consider the edges $\mathcal{E}_v$). But by minimizing this cost function $\phi_{\theta_e} \left( X \right)$ with respect to $\theta_e$, all other trainable parameters except $\theta_e$ remain fixed and are not changed. 

    \subsection{Discrete \graphPINN\ approach}
	\label{subsec:discrete_graphPINN}

    While the approaches discussed previously aim to compute the solution on the time-space domain in one shot, this approach is based on a time-discretization of the underlying parabolic PDE.
    Here we use the fully-implicit Euler time-stepping scheme~\eqref{eq:strong_pde} introduced in \cite{raissi2019physics}. 

    To comply with our previous notation, we assume an equidistant time grid $0=t_0<t_1<\ldots<t_{n_t}=T$ with grid size $\tau = t_n - t_{n-1}$, $n=1,\ldots,n_t$. Note that we chose the same $\tau$ on each edge.
    For a fixed time $t_n$, $n=0,\ldots,n_t$, the function $\rho_e^n(x) \coloneqq \rho_e(t_n, x)$ represents the density in a spatial point $x \in [0,\ell_e]$ at time $t_n$.
    The new residual network is given by
\begin{equation}
	\label{eq:residual_network_discrete}
	r^n_{\theta_e} \left( x \right)
    \coloneqq
    \frac{\rho^n_{\theta_e} \left( x \right) -
    \rho^{n-1}_{\theta_e} \left( x \right)}{\tau}
    - \partial_x   \left(  \varepsilon \partial_x  \rho^n_{\theta_e} \left( x \right) - f \left( \rho^n_{\theta_e} \left( x \right) \right) \partial_x V \left( t_n, x \right) \right),
\end{equation}
which replaces the continuous-time residual network~\eqref{eq:residual_network}.
This corresponds to a semi-discrete approach, where the time derivative is approximated by a simple finite difference, while the remaining (spatial) derivatives are still computed by automatic differentiation. 
To compute the unknown density $\rho_{\theta_e}^n$, we include the remaining misfit terms for initial, vertex and boundary conditions as in the \graphPINN \ approach, yet evaluated at the fixed current time step $t^n$.
For this we choose a set of equidistant points $\left\{ x^i_e \right\}_{i=0}^{n_0}  \subset \left[ 0, \ell_e \right]$ on each edge $e \in \mathcal{E}$ with $0 = x^0_e < x^1_e < \ldots < x^{n_0}_e = \ell_e$.
The collocation points used at time step $t=t^n$ are $X_e = \left\{ \left( t^n, x_e^i \right) \right\}_{i=0}^{n_t} \subset \{t^n\} \times \left[0, \ell_e\right]$ for the residual misfit term defined by \eqref{eq:misfit_residual}, while the Kirchhoff-Neumann and Dirichlet vertex conditions are evaluated at the origin and terminal vertices of each edge, resp.
Note that for the computation of the solution at $t = t^1$, the values of the previous density $\rho^0_{\theta_e}(x)$ are given by the initial condition~\eqref{eq:initial_conditions}.

	\begin{figure}
		\includegraphics[width=.7\textwidth]{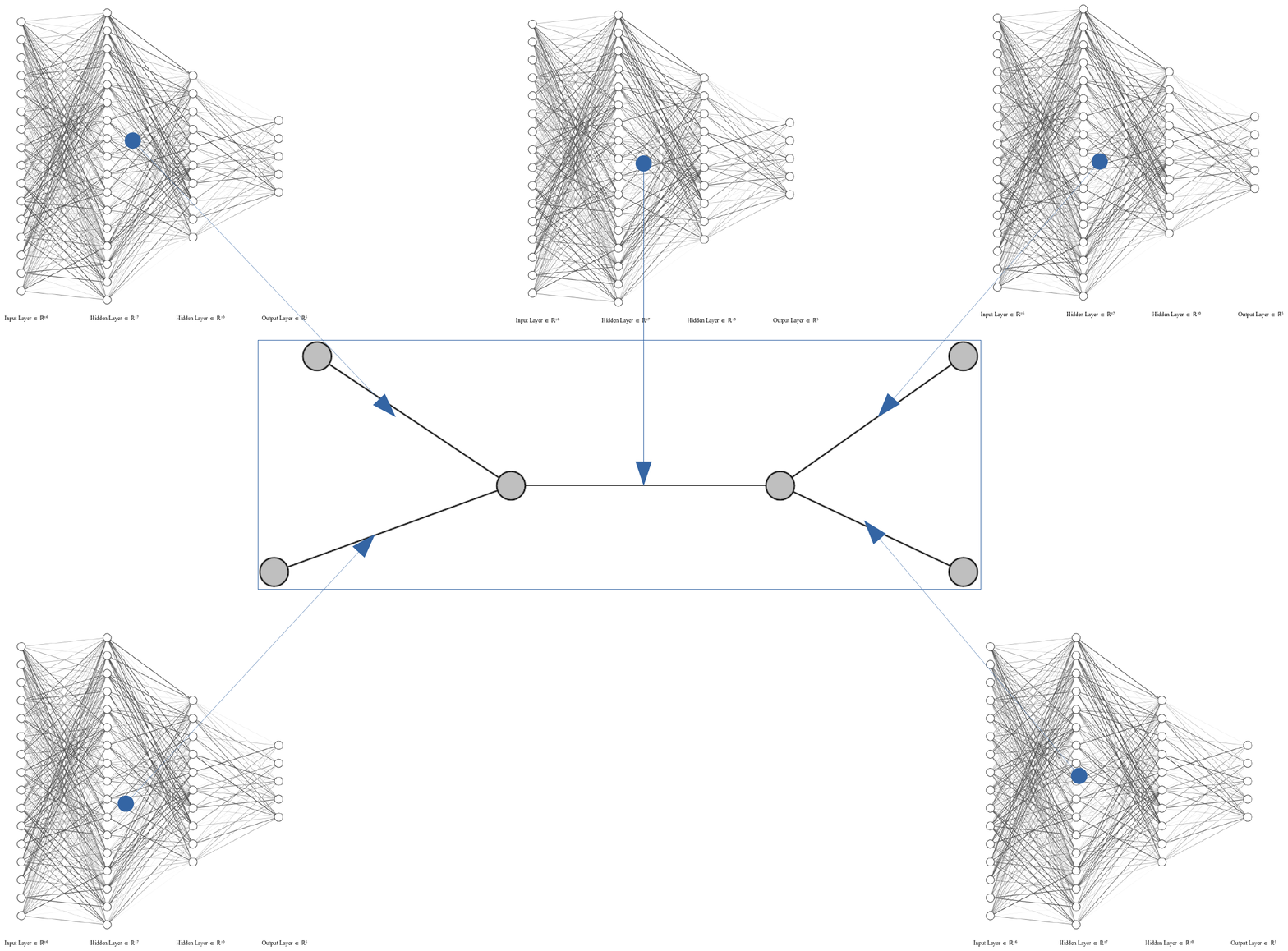}
		\caption{Simple network with an FNN associated with each of its edges.}
	   \end{figure}

	\section{A finite volume method}
	\label{sec::fvm}
We want to compare the method studied in the previous sections with
a traditional numerical approach, namely a finite volume method.
Throughout this section we assume $V_e$ to be affine linear and
write $d_e = \partial_x V_e$. The edge set incident to a vertex $v\in \V$ is
denoted by $\E_v$ and we distinguish among $\E_v^{\text{in}}\:=\{e\in \E\colon
e=(\widetilde v,v)\ \text{for some}\ \widetilde v\in \V\}$ and
$\E_v^{\text{out}} = \E_v\setminus \E_v^{\text{in}}$.
The control volumes are defined as follows. To each edge $e\in
\E$ we associate an equidistant grid of the parameter domain
\begin{equation*}
    0 = x^e_{-1/2} < x^e_{1/2} <\ldots < x^e_{n_e+1/2} = L_e
\end{equation*}
with $h_e=x^e_{k+\frac12} - x^e_{k-\frac12}$, and introduce
the intervals $I_k^e = (x_{k-1/2}, x_{k+1/2})$ for all
$k=0,\ldots,n_e$.
We introduce the following control volumes for our finite volume method,
\begin{itemize}
    \item the interior edge intervals $I_1^e,\ldots,I_{n_e-1}^e$ for
        each $e\in \E$,
    \item the vertex patches $I^v = \big(\cup_{e\in \E_v^{\text{in}}} I_{n_e}^e\big)
        \cup \big(\cup_{e\in \E_v^{\text{out}}} I_0^e\big)$ for each $v\in \V$.
\end{itemize}
A semi-discrete approximation of the problem
\eqref{eq:strong_pde}--\eqref{eq:Dirichlet_conditions}
can be expressed by the volume averages
\begin{align*}
    \rho_k^{e}(t) &= |I_k^e|^{-1}\int_{I_k^e} \rho_e(t,x)\d x,\\
    \rho^{v}(t) &= |I^v|^{-1} \Big(
    \sum_{e\in \E_v^{\text{out}}} \int_{I^e_0}\rho_e(t,x)\d x +
    \sum_{e\in \E_v^{\text{in}}} \int_{I^e_{n_e}}\rho_e(t,x)\d x
    \Big),
\end{align*}
for all $e\in \E$, $k=1,\ldots,n_e-1$, resp.\ $v\in \V$.
With the definition of the vertex patches we strongly enforce the continuity
in the graph nodes. Integration over some interval $I^e_k$,
$k=0,\ldots,n_e$, $e\in \E$, gives
\begin{align}
    \label{eq:finite_volume_integral}
    \int_{I_k^e}\,\partial_t \rho_e(t,x) \d x
    &=
    \int_{I_k^e} \partial_x (\varepsilon\,\partial_x\rho_e(t,x) -
    f(\rho_e(t,x))\,d_e(t))\d x \nonumber\\
    =h_e\,\partial_t \rho_k^e &=
    \left(
    \varepsilon\,\partial_x\rho_e(t,x) -
    f(\rho_e(t,x))\,d_e(t)
    \right)\Big\vert_{x^e_{k-1/2}}^{x^e_{k+1/2}}.
\end{align}
The diffusive fluxes are approximated by central differences
\begin{equation*}
    \partial_x\rho(t,x^e_{k+1/2}) \approx \frac1{h_e}(\rho_{k+1}^e(t)-\rho_k^e(t))
\end{equation*}
and for the convective fluxes we use, for stability reasons, the
Lax-Friedrichs numerical flux
\begin{align}\label{eq:lax_friedrichs_flux}
    f(\rho_e(t,x_{k+1/2}))\,d_e(t) &\approx F^e_{k+1/2}(t)\nonumber\\
    &:= \frac12 (f(\rho_k^e(t)) + f(\rho_{k+1}^e(t)))\,d_e(t)
    - \frac{\alpha}2 (\rho_{k+1}^e(t) - \rho_k^e(t)),
\end{align}
where we use the convention $\rho_0^e = \rho^v$ for $v\in V$
satisfying $e\in \E_v^{\text{out}}$ and $\rho_{n_e}^e =
\rho^{\widetilde v}$ with $\widetilde v\in V$ satisfying $e\in
\E_{\widetilde v}^{\text{in}}$.
The parameter $\alpha>0$ is some stabilization parameter, chosen
sufficiently large. At inflow and outflow vertices $v\in
\V_\D$ we insert the boundary condition
\eqref{eq:Dirichlet_conditions} into \eqref{eq:finite_volume_integral}
and obtain
\begin{equation*}
    \sum_{e\in \E_v} \left(\varepsilon\,\partial_x\rho_e(t,v) -
    f(\rho_e(t,v))\,d_e(t)\right)
    \approx -\alpha_v(t)\,(1-\rho^v) + \beta_v(t)\,\rho^v.
\end{equation*}
Combining the previous investigations gives the following set of
equations for each control volume $I_k^e$, $k=1,\ldots,n_e-1$, $e\in \E$, and $I^v$, $v\in \V$,
respectively.
\begin{subequations}
    \label{eq:semi_discrete_fvm}
    \begin{align}
        \intertext{For each $e\in \E$ and $k=1,\ldots,n_e-1$:}
        h_e\,\partial_t \rho_k^e(t) + \varepsilon\,\frac{-\rho_{k-1}^e(t) +
        2\rho_k^e(t) - \rho_{k+1}^e(t)}{h_e} - F_{k-\frac12}^e(t) + F_{k+\frac12}^e(t) &= 0.
        \\
        \intertext{For each $v\in \V_\K$:}
        \sum_{e\in \E_v} h_e\,\partial_t\rho^v(t)
        + \sum_{e\in \E_v^{\text{in}}}
        \left(\varepsilon\,\frac{\rho^v(t)-\rho_{n_e-1}^e(t)}{h_e} -
        F^e_{n_e-\frac12}(t)\right) &\nonumber\\
        - \sum_{e\in \E_v^{\text{out}}}
        \left(\varepsilon\,\frac{\rho_1^e(t)- \rho^v(t)}{h_e} - F^e_{\frac12}(t)\right)
        &= 0.\label{eq:fvm_kirchhoff_vertices}\\
        \intertext{For each influx node $v\in \V_\D^{\text{in}}$:}
        \sum_{e\in \E_v} h_e\,\partial_t\rho^v(t)
        - \sum_{e\in \E_v^{\text{out}}}
        \left(\varepsilon\,\frac{\rho_1^e(t)-\rho^v(t)}{h_e} -
        F^e_{\frac12}(t)\right)
        -\alpha_v\,(1-\rho^v(t))
        &= 0.
        \intertext{For each outflux node $v\in \V_\D^{\text{out}}$:}
        \sum_{e\in \E_v} h_e\,\partial_t\rho^v(t)
        + \sum_{e\in \E_v^{\text{in}}}
        \left(\varepsilon\,\frac{\rho^v(t)-\rho_{n_e-1}^e(t)}{h_e} -
        F^e_{n_e-\frac12}(t)\right)
        + \beta_v\,\rho^v(t)
        &= 0.
    \end{align}
\end{subequations}
In \eqref{eq:fvm_kirchhoff_vertices} accumulated contributions
evaluated in $v$ vanish due to the Kirchhoff-Neumann vertex conditions
\eqref{eq:Kirchhoff_Neumann_condition}.

To solve the system of ordinary differential equations
\eqref{eq:semi_discrete_fvm} for the unknowns
$\rho^e_k$ and $\rho^v$, respectively, we introduce the following
time-discretization.
For some equidistant time grid $0=t_0<t_1<\ldots<t_{n_t}=T$ with grid
size $\tau = t_n - t_{n-1}$, $n=1,\ldots,n_t$, we define the following
grid functions by
\begin{equation*}
    \rho^{v,n} = \rho^v(t_n),\quad \rho^{e,n}_k = \rho^e_k(t_n),
    \quad F_{k+1/2}^{e,n} = F_{k+1/2}^e(t_n).
\end{equation*}
We restrict the equations \eqref{eq:semi_discrete_fvm} to the grid points
and replace the time derivative by a difference quotient, evaluate
the diffusion terms in $t_{n+1}$ and the convective terms in $t_n$.
This yields for each $n=1,\ldots,n_t$
the following system of equations:
\begin{subequations}
    \label{eq:fully_discrete_fvm}
    \begin{align}
        \intertext{For each $e\in \E$ and $k=1,\ldots,n_e-1$:}
        h_e\,\rho_k^{e,n} + \varepsilon\,\tau\,\frac{-\rho_{k-1}^{e,n} +
        2\rho_k^{e,n} - \rho_{k+1}^{e,n}}{h_e}
        &= h_e\,\rho_k^{e,n-1}
        + \tau\left(F_{k-\frac12}^{e,n-1} -
	F_{k+\frac12}^{e,n-1}\right).
	\label{fully_discrete_fvm_edge_equation}
        \\
        \intertext{For each $v\in \V_\K$:}
        \abs{I_v}\,\rho^{v,n}
        + \tau\,\varepsilon\,\sum_{e\in \E_v^{\text{in}}}
        \,\frac{\rho^{v,n}-\rho_{n_e-1}^{e,n}}{h_e}
        &- \tau\,\varepsilon\,\sum_{e\in \E_v^{\text{out}}}
        \frac{\rho_1^{e,n}- \rho^{v,n}}{h_e} \nonumber\\
        = \abs{I_v}\,\rho^{v,n-1} + \tau\,\sum_{e\in \E_v^{\text{out}}}F^{e,n-1}_{\frac12}
        &- \tau\,\sum_{e\in \E_v^{\text{in}}}F^{e,n-1}_{n_e-\frac12}.\\
        \intertext{For each influx node $v\in \V_\D^{\text{in}}$:}
        \abs{I_v}\,\rho^{v,n}
        - \tau\,\varepsilon\sum_{e\in \E_v^{\text{out}}}
        \frac{\rho_1^{e,n}-\rho^{v,n}}{h_e}
        &= \abs{I_v}\,\rho^{v,n-1} +
        \tau\,F^{e,n-1}_{\frac12}\nonumber\\
        &\quad+ \tau\,\alpha_v\,(1-\rho^{v,n-1}). \\
        \intertext{For each outflux node $v\in \V_\D^{\text{out}}$:}
        \abs{I_v}\,\rho^{v,n}
        + \tau\,\varepsilon\sum_{e\in \E_v^{\text{in}}}
        \frac{\rho^{v,n}-\rho_{n_e-1}^{e,n}}{h_e}
        &= \abs{I_v}\,\rho^{v,n-1} + \tau\,\sum_{e\in
        \E_v^{\text{in}}} F^{e,n-1}_{n_e-\frac12} \nonumber\\
        &\qquad-\tau\,\beta_v\,\rho^{v,n-1}.
    \end{align}
\end{subequations}
The initial data are established by
\begin{equation*}
    \rho_k^{e,0}=\pi_{I_k^e}(\rho_0),\qquad \rho^{v,0} = \pi_{I^v}(\rho_0),
\end{equation*}
where $\pi_M$ denotes the $L^2$-projection onto the constant functions on
a subset $M\subset \Gamma$.
Note that this set of equations is linear in the unknowns in the new time
point $\rho_k^{e,n}$, $k=1,\ldots,n_e-1$, $e\in \E$ and $\rho^{v,n}$,
$v\in \V$.
The fully-discrete approximation $\rho_{\tau\,h}\colon [0,T]\times \Gamma\to
\R$ then reads
\begin{equation*}
    \widetilde \rho(t,x) = \widehat\rho^n(x)\quad \text{for}\ t\in
    [t_n,t_{n+1}),
\end{equation*}
with
\begin{equation*}
    \widehat\rho^n(x) = \rho^{v,n}\ \text{for}\ x\in I^v,\qquad \widehat\rho^n(x)
    = \rho_k^{e,n}\ \text{for}\ x\in I_k^e.
\end{equation*}

It is well-known that finite-volume schemes like \eqref{eq:fully_discrete_fvm}
guarantee a couple of very important properties. On the one hand, there is a
well established convergence theory, see e.\,g.\
\cite{MortonStynesSuli1997,LazarovMishevVassilevski1996,ThijeBoonkkampAnthonissen2010}.
On the other hand, our scheme is
mass-conserving and bound-preserving which we show in the following theorem.
Thus, the finite volume approach is suitable of designing reference solutions used to
measure the quality of the solutions obtained with the approaches from Section~\ref{sec::pinns}. 
\begin{theorem}
    The solution of \eqref{eq:fully_discrete_fvm}, $\widetilde\rho$, satisfies the following
    properties:
    \begin{enumerate}[label=\roman*)]
        \item The scheme is mass conserving, i.e., if $\alpha_v\equiv \beta_v\equiv 0$ for all $v\in \V_{\D}$,
            then there holds
            \begin{equation*}
                \int_\Gamma\widehat\rho^n\d x = \int_\Gamma\widehat\rho^0\d x\qquad\forall n=1,\ldots,n_t.
            \end{equation*}
        \item Assume that $f(x) = x(1-x)$ and in
	\eqref{eq:lax_friedrichs_flux} choose $\alpha=1$. Then, the scheme is
            bound-preserving, i.e., there holds
            \begin{equation*}
                \widetilde\rho(t,x)\in [0,1]\qquad \forall t\in [0,T],
                x\in \Gamma,
            \end{equation*}
	    provided that $\tau\le \min_{e\in \E}h_e$.
    \end{enumerate}
\end{theorem}
\begin{proof}
    \begin{enumerate}[label=\roman*)]
        \item This directly follows after summing up all the equations in
            \eqref{eq:fully_discrete_fvm}. Note that the diffusive and convective
            fluxes cancel out.
        \item The system \eqref{eq:fully_discrete_fvm} can be written as a
            system of linear equations of the form
            \begin{equation}
                \label{eq:fvm_equation_system}
                (M+\tau\,\varepsilon A)\vec \rho^n = M\vec \rho^{n-1} + \vec F(\vec\rho^{n-1}),
            \end{equation}
            where $M$ is the mass matrix and $A$ contains the coefficients of the
            diffusion terms on the left-hand side of
	    \eqref{eq:fully_discrete_fvm}.
	    The vector $\vec \rho^n$ contains the unknowns
            $\rho^{v,n}$ and $\rho_k^{e,n}$. In the usual ordering the
	    unknowns and equations the matrix $M+\tau\,\varepsilon\,A$
	    is strictly diagonal dominant and is thus an M-matrix.
            The inverse possesses non-negative entries only.
	    The right-hand side of \eqref{eq:fvm_equation_system} is
	    also non-negative under the assumption $\vec\rho^{n-1}\in
	    [0,1]$. We demonstrate this for the equation
	    \eqref{fully_discrete_fvm_edge_equation}. Insertion
	    of \eqref{eq:lax_friedrichs_flux} and reordering the terms
	    yields
	    \begin{align*}
		&h_e\,\rho_k^{e,n-1} + \tau\left(F_{k-1/2}^{e,n-1} -
		F_{k+1/2}^{e,n-1}\right) = (h_e-\alpha\,\tau)\,\rho_k^{e,n-1}\\
		&\quad 
		+\frac\tau2\left((1-\rho_{k-1}^{e,n-1}) + \alpha\right)
		\rho_{k-1}^{e,n-1}
		+
		\frac\tau2\left(-(1-\rho_{k+1}^{e,n-1}) + \alpha\right)
		\rho_{k+1}^{e,n-1} \ge 0.
	    \end{align*}
	    The non-negativity follows from $\rho_{k}^{e,n-1}\in
	    [0,1]$ for $k=0,\ldots,n_e$ and $\alpha=1$ as well as
	    $\tau\le \min_{e\in \E}h_e$.
	    This, together with the M-matrix property of
            $M+\tau\,\varepsilon\,A$, implies $\vec\rho^n\ge 0$.

	    Due to
            $f(x)=x(1-x)$ we may rewrite
            \eqref{eq:fvm_equation_system} in the form
            \begin{equation*}
                (M+\tau\,\varepsilon A)(\vec1-\vec \rho^n)
                = M(\vec1-\vec\rho^{n-1}) + \vec G(\vec 1-\vec\rho^{n-1}),
            \end{equation*}
            with some vector-valued function $\vec G$. With similar
	    arguments like before we conclude that the right-hand side
	    is non-negative and	    
            thus, $1-\vec\rho^n \ge 0$, which proves the upper bound.
	    By induction the result follows for all $n=1,\ldots,n_t$.
    \end{enumerate}
\end{proof}

\section{Numerical results}
\label{sec::results}
	
We want to compare the various \pinn\ approaches presented earlier on a representative model graph. We here focus on the accuracy of the approximation in comparison to the finite volume method presented in \eqref{sec::fvm} and have not yet tailored the approaches for computational speed. The accuracy is validated by numerical experiments in which we measure the deviation of the values generated by the different \pinn\ approaches using different neural networks from the values generated by the \fvm. These numerical experiments are implemented with \lstinline!Python 3.8.8!, rely on the \textsc{TensorFlow} implementation and were run on a Lenovo ThinkPad L490, 64 bit Windows system, 1.8 Ghz Intel Core i7-8565U and 32 GB RAM. 

	\subsection{Two different ways to utilize neural networks }
	Before we proceed to the numerical experiments, we specify the neural networks and their topology that were used for the approximations $\left \{ \rho_{\theta_e} \right \}_{e \in \mathcal{E}}$. Of course, we have an infinite freedom of choice for choosing neural networks in our \pinn\ approaches, the only constraints being the dimension of the input $\left(t, x\right) \in \mathbb{R}^2$, the dimension of the output $\rho_{\theta_e}\left(t, x\right) \in \mathbb{R}$ on each edge $e \in \mathcal{E}$ and the differentiability of the corresponding neural network up to order two (due to $\partial_{xx} \rho_{\theta_e}$ appearing in \eqref{eq:strong_pde}).

	Since we have to find an approximation $\rho_{\theta_e}$ for each edge $e \in \mathcal{E}$, it is quite evident to use one single neural network for the approximation of $\rho_e$ on one individual edge $e$. This intuitive choice leads to $\theta_e$ describing the trainable parameters of a single neural network, i.e. of $\rho_{\theta_e}$, as also previously assumed, and results in the fact that we have to train as many neural networks as we have edges of the graph. The use of a single neural network for approximating the solution $\rho_e$ on a single edge $e$ allows both the type of neural network used and the hyperparameters of a single network such as activation function, depth or width of the network to be freely chosen, depending on intuitive knowledge about the regularity of the solution on each edge. In the case of smooth edges, a shallow network can be used, while a deep neural network can be used on an edge where a complex solution is expected. 
	The easiest choice is to use a feed-forward neural network with $L$-layers for each edge $e \in \mathcal{E}$, which we denote by $\operatorname{fnn}_{\theta_e}$ and is defined by 
	\begin{equation} 
		\label{one_for_each}
		\begin{gathered}
			\operatorname{fnn}_{\theta_e} \colon \mathbb{R}^2 \to \mathbb{R}, \\
			\\
			\operatorname{fnn}_{\theta_e}\left(t, x\right) = \sigma_L\left(W^L_e \sigma_{L-1}\left(W^{L-1}_e\sigma_{L-2}\left(\cdots \sigma_{1}\left(W^{1}_e x^0 +b^1_e\right) \cdots\right) + b^{L-1}_e\right) + b^{L}_e\right),
		\end{gathered} 
	\end{equation} 
	where $x^0 = \left(t, x\right)^{\mathrm{T}} \in \left(0, T\right) \times \left[0, \ell_e\right] \subset \mathbb{R}^2$, $\sigma_l \colon \mathbb{R}^{n_l} \to \mathbb{R}^{n_l}$ is a non-linear activation function and $\theta_e = \left\{ \left\{ W^l_e \right\}_{l = 1, \ldots, L}, \left\{ b^l_e \right\}_{l = 1, \ldots, L} \right\}$ with $W^l_e \in \mathbb{R}^{n_l \times n_{l-1}}$ and $b^l_e \in \mathbb{R}^{n_l}$, where $n_0 = 2$ and $n_L = 1$. All activation functions $\left \{ \sigma_l \right \}_{l = 1, \ldots, L}$ must be differentiable up to order two so that the resulting network  is also differentiable up to order two with respect to the input $t \in \left(0, T \right)$ and $x \in  \left[0, \ell_e\right]$. Feed-forward neural networks have already been used in various \pinn \ setups, so this choice can be considered as reasonable. 

	Another possibility is to use just one single \textsc{FNN} with $L$ layers for the approximations on all edges of the graph at the same time. This means that this single \textsc{FNN} returns for an input $\left(t,x\right) \in \mathbb{R}^2$ the values $\left\{ \rho_{\theta_{e_i}}\left(t, x\right) \right\}_{i = 1, \ldots, E}$ for all edges $\mathcal{E} = \left\{ e_i \right\}_{i = 1, \ldots, E}$, where $E = \abs{\mathcal{E}}$. This is achieved by choosing the number of neurons in the output/last layer equal to the number of edges of the graph, i.e. $n_L = E$. We define this neural network as follows
	\begin{equation} 
		\label{one_for_all}
		\begin{gathered}
			\operatorname{FNN}_{\hat{\theta}} \colon \mathbb{R}^2 \to \mathbb{R}^E \\
			\\
			\operatorname{FNN}_{\hat{\theta}}\left(t, x\right) = \sigma_L\left(W^L \sigma_{L-1}\left(W^{L-1}\sigma_{L-2}\left(\cdots \sigma_{1}\left(W^{1}x^0 +b^1\right) \cdots\right) + b^{L-1}\right) + b^{L}\right),
		\end{gathered} 
	\end{equation}
	where $x^0 = \left(t, x\right)^{\mathrm{T}} \in \left(0, T\right) \times \left[0, \ell_e\right] \subset \mathbb{R}^2$ and $\hat{\theta} = \left\{ \left\{ W^l \right\}_{l = 1, \ldots, L}, \left\{ b^l \right\}_{l = 1, \ldots, L} \right\}$ with $W^l \in \mathbb{R}^{n_l \times n_{l-1}}$ and $b^l \in \mathbb{R}^{n_l}$, where $n_0 = 2$ and $n_L = E$. The individual approximations are given by 
	\begin{equation*}
		\rho_{\theta_{e_i}}\left(t, x\right) = \left[ \operatorname{FNN}_{\hat{\theta}}\left(t, x\right) \right]_i \in \mathbb{R},
	\end{equation*}
	i.e. the $i$-th entry of the networks output $\operatorname{FNN}_{\hat{\theta}}\left(t, x\right) \in \mathbb{R}^E$ is the approximation of the solution on the $i$-th edge. 
	In this work we use this neural network with multidimensional output in our numerical experiments only for the continuous \graphPINN\ approach, because it turned out to be difficult to assign the trainable parameters to the individual edges and to use automatic differentiation for the derivatives on the individual edges in a practical way. This means the cost function given by \eqref{eq:loss:1} is minimized with respect to the trainable parameters $\hat{\theta}$ of this single network $\operatorname{FNN}_{\hat{\theta}}$ and if \eqref{eq:misfit_continuity} is to be used for the continuity misfit term, additionally the trainable parameters $\left\{ \rho_{v}^i \right\}_{i=1}^{n_b}$.

    The use of a single neural network is motivated by the hope that the neurons in the hidden layers learn the structure of the entire graph and the resulting communication of the edges via the vertices, i.e. that all interactions within the graph are taken into account by the neurons after the learning phase. Furthermore, we hope that the computational cost is reduced, since only the weights and biases of this single network need to be trained, provided the depth and width of the neural network $\operatorname{FNN}_{\hat{\theta}}$ are not too large. The fact that $\operatorname{FNN}_{\hat{\theta}}\left(t, x\right)$ generates the output for all edges at the same time is also an advantage, as in an implementation only the execution of one network is necessary instead of several. However, we point out that this approach is currently implemented for the case if all edges of the graph have the same length, because then the same collocation points $X_e = \left\{ \left( t_e^i, x_e^i \right) \right\}_{i=1}^{n_e}$ and $X_{e,0} = \left\{x_{e,0}^i \right\}_{i=1}^{n_0}$ can be used for the misfit terms given by \eqref{eq:misfit_residual} and \eqref{eq:misfit_initial} for all edges $e \in \mathcal{E}$. Adding a scaling layer should also enable the use of different edge lengths.

	The hyperbolic tangent is used component-wise as activation function for each hidden layer, as well as for the output layer, of all networks used in the numerical experiments, i.e.
	\begin{equation}
		\begin{gathered}
			\label{eq:tanh}
			\sigma_l \colon \mathbb{R}^{n_l} \ni x \mapsto \sigma_l (x) \coloneqq \left(
				\begin{array}
					{c} \sigma_l \left( x_{1} \right) \\
					\vdots \\
					\sigma_l \left( x_{n_l} \right)
				\end{array}
				\right) \in \mathbb{R}^{n_l},  \\
				\\
			\sigma(x_i) =\tanh (x_i)=\frac{\exp (x_i)-\exp (-x_i)}{\exp (x_i)+\exp (-x_i)} \in \mathbb{R}.
		\end{gathered}
	\end{equation}
	This activation function is twice continuously differentiable, which means that the resulting neuronal networks are twice continuously differentiable.
    {We also conducted experiments with the sigmoid activation function, which seems to be reasonable in particular in the output layer due to property that the solution satisfies $0 \le \rho_e \le 1$ a.e.\ on each edge $e \in \E$. The obtained approximation errors were most often slightly worse than the ones obtained for the hyperbolic tangent activation function.}

	\subsection{Numerical experiments}

	Next, we present the results of our numerical experiments in which we measured the deviation of the values generated by the neural networks trained by each \pinn\ approach from the values generated by the \fvm\ at the same grid points. In order to compare the results of all numerical experiments, we use the same problem setup for each experiment, i.e. the set of drift-diffusion equations is considered on the same model metric graph $\Gamma$ under the same initial and vertex conditions in each experiment. The structure of this model metric graph is illustrated in \eqref{fig7}. 

	\begin{figure}
		\begin{center}
			\begin{tikzpicture}
				\node[shape=circle,draw=black] (v1) at (-2.4,1.4) {$v_1$};
				\node[shape=circle,draw=black] (v5) at (2.4,1.4) {$v_5$};
				\node[shape=circle,draw=black] (v3) at (-1,0) {$v_3$};
				\node[shape=circle,draw=black] (v4) at (1,0) {$v_4$};
				\node[shape=circle,draw=black] (v2) at (-2.4,-1.4) {$v_2$};
				\node[shape=circle,draw=black] (v6) at (2.4,-1.4) {$v_6$};
				
				\path [->](v1) edge node[above] {$e_1$} (v3);
				\path [->](v2) edge node[below] {$e_2$} (v3);
				\path [->](v3) edge node[above] {$e_3$} (v4);
				\path [->](v4) edge node[above] {$e_4$} (v5);
				\path [->](v4) edge node[below] {$e_5$} (v6);
			\end{tikzpicture}
		\end{center}
		\caption{The model metric graph $\Gamma$ used in the numerical experiments.}
		\label{fig7}
	\end{figure}

	The model graph $\Gamma$ consists of $6$ vertices $\mathcal{V} = \left\{ v_i \right\}_{i = 1,\ldots, 6}$ and $5$ edges $\mathcal{E} = \left\{ e_i \right\}_{i = 1,\ldots, 5}$ and is of course directed. The vertices $v_1$, $v_2$, $v_5$ and $v_6$ are the exterior vertices, i.e. $\mathcal{V}_{\mathcal{D}} = \left\{v_1, v_2, v_5, v_6 \right\}$, and the vertices $v_3$ and $v_4$ are the interior vertices, i.e. $\mathcal{V}_{\mathcal{K}} = \left\{v_3, v_4 \right \}$, of the graph $\Gamma$. 

	As an approximation problem, we consider the set of drift-diffusion equations defined by \eqref{eq:strong_pde} on all edges $e \in \mathcal{E}$ of the graph $\Gamma = \left(\mathcal{V}, \mathcal{E} \right)$, whose structure is given by \eqref{fig7}, under the initial and vertex conditions given by \eqref{eq:Kirchhoff_Neumann_condition}, \eqref{eq:continuity_condition}, \eqref{eq:Dirichlet_conditions} and \eqref{eq:initial_conditions}. Furthermore we have the following assumptions:

	\begin{enumerate}
		\item The observation time ends at $T = 10$ and all edges of the graph $\Gamma$ are of equal length with $\ell = 1$.
		\item Let $\varepsilon = 0.01$ in \eqref{eq:strong_pde}.
		\item The mobility in \eqref{eq:strong_pde} is given by $f \left( \rho_e(t,x) \right) = \rho_e \left(t,x\right)\left(1-\rho_e \left(t,x\right)\right)$.
		\item $\partial_x V_e \left(t,x\right) = 1$ for all $\left(t,x\right) \in \left(0, 10\right) \times \left[0,1\right]$ holds in \eqref{eq:strong_pde} for the potential $V_e$ of each edge $e \in \mathcal{E}$. 
		\item On the exterior vertices $\mathcal{V}_{\mathcal{D}}$, we have for the flux vertex conditions, \eqref{eq:Dirichlet_conditions}, constant influx and outflux rates which are specified for $v_1$ by $\alpha_{v_1}\left(t\right) = 0.9$ and $\beta_{v_1}\left(t\right) = 0$, for $v_2$ by $\alpha_{v_2}\left(t\right) = 0.3$ and $\beta_{v_2}\left(t\right) = 0$, for $v_5$ by $\alpha_{v_5}\left(t\right) = 0$ and $\beta_{v_5}\left(t\right) = 0.8$ and for $v_6$ by $\alpha_{v_6}\left(t\right) = 0$ and $\beta_{v_6}\left(t\right) = 0.1$.
		\item The initial conditions are $\rho_e\left(0,x\right) = 0$ for all $x \in \left[0, \ell_e\right]$ for each edge $e \in \mathcal{E}$. 
	\end{enumerate}

	As already mentioned, to compare the accuracy we use as reference a \fvm\ solution as described in \eqref{sec::fvm}, which was obtained on an equidistantly spaced time and space grid with $n_e=\num{8000}$ for all $e \in \E$ and $n_t=\num{4000}$.
    The chosen stabilization parameter is $\alpha = 1$.

    We measure the approximation error along one edge $e \in \E$ in the $L^2$-norm defined by
	\begin{equation*}
		\lVert \varphi_h \rVert^2_{L^2(e)} \coloneqq \sum^{n_x}_{i = 0} \ \sum^{n_t}_{j = 0} \ \Delta_x \ \Delta_t \ \varphi_h(t_j, x_i)^2 
	\end{equation*}
    for a grid function $\varphi_h$, while the $L^2$-norm on a graph $\Gamma$ is given by 
	\begin{equation}
		\label{eq:L2_norm}
		\lVert \varphi_h \rVert^2_{L^2( \Gamma )} \coloneqq \sum_{e \in \mathcal{E}} \ \lVert \varphi_h \rVert^2_{L^2(e)}.
	\end{equation}

	We have performed five numerical experiments of which the resulting $L^2$-error, computed using \eqref{eq:L2_norm}, of the values generated by the respective \pinn\ approach using the respective neural network with respect to the values generated by the \fvm. 

    While a hyperparameter optimisation for our implementations would have been possible,
    we decided to test the implementations on a set of different numbers of hidden layers and different numbers of neurons in the hidden layers for better comparability.
    In all tables, the columns specify the number of hidden layers and the rows the number of neurons per hidden layer, from which the topology of the used neural network can be derived.
	The results of the first two numerical experiments are shown in Table~\ref{tab:SpaceTime_each_edge} and Table~\ref{tab:one_net}, where the relative $L^2$-errors of the values generated with the (continuous) \graphPINN\ approach, as described in Section~\ref{subsec:graphPINN}, are listed. For the results in Table~\ref{tab:SpaceTime_each_edge}, one feed-forward neural network was used for each edge of the graph, as described by \eqref{one_for_each}, and for the results in Table~\ref{tab:one_net} one single feed-forward neural network was used for all edges of the graph, as described by \eqref{one_for_all}.
    As mentioned before, the hyperbolic tangent given by \eqref{eq:tanh} was used as activation function for each layer of all used neural networks.
	For the minimization of the cost function $\Phi_{\theta}$ in these two numerical experiments we use a combination of \adam\ and \lbfgs\ optimizers as usual for the \pinn\ setup.
    First a variant of the gradient descent method is used, in which the direction is given by the \adam\ optimizer from the module \lstinline!Keras!~\cite{Chollet:2015}, which belongs to the package \lstinline!Tensorflow!~\cite{tensorflow2015-whitepaper}. We changed the input parameter \lstinline!learning_rate!, which obviously specifies the learning rate of the method, to \lstinline!0.01! and left all other input parameters of the \adam\ optimizer by default. In both numerical experiments this method stops after $1000$ iterations. After that, we use the \lbfgs\ optimizer from the package \lstinline!SciPy!, see \cite{SciPy:2020}, and we set as parameters in the options \lstinline!maxfun = 500000! which specifies the maximum number of function evaluations, \lstinline!maxcor = 50! which specifies the maximum number of stored variable metric corrections used in the L-BFGS approximation of the Hessian of the cost function $\Phi_{\theta}$, \lstinline!maxls = 50! which specifies the maximum number of line search steps per iteration and \lstinline!ftol = 1.0*np.finfo(float).eps! which specifies a parameter in a stopping criterion that aborts the method if the values of the cost function $\Phi_{\theta}$ between two iterates are too small. 
	For the first numerical experiment, which produced the values in Table~\ref{tab:SpaceTime_each_edge}, we set \lstinline!maxiter = 20000! which specifies the maximum number of iterations, and for the second numerical experiment, which produced the values in Table~\ref{tab:one_net}, we set \lstinline!maxiter = 50000!.

	For the third experiment, we evaluated the relative $L^2$-error of the values generated with the discrete \graphPINN\ approach, as described in Section~\ref{subsec:discrete_graphPINN}, where one feed-forward neural network was used for each edge of the graph, as described by \eqref{one_for_each}.
	Our numerical results confirm that the approach is similar to the first continuous space-time approach, since it corresponds to a rearrangement of the terms in the cost function and an alternating optimization procedure.
	In the case of each neural network consisting of $3$ hidden layers and $20$ neurons per hidden layer, the relative $L^2$ error obtained after $\num{20000}$ \adam\ iterations with a step-wise decreasing learning rate (no \lbfgs\ iterations) is $0.623$.

	For these first $3$ numerical experiments the collocation points were randomly generated using \lstinline!tf.random.uniform!, which is provided by the package \lstinline!Tensorflow!, see \cite{tensorflow2015-whitepaper}, and uses a uniform distribution.
    As time-space collocation points $X_e$ for \eqref{eq:misfit_residual} we used $n_e = 4000$ randomly selected points $\left\{ \left( t_e^i, x_e^i \right) \right\}_{i=1}^{n_e} \subset \left( 0, T \right) \times \left[0, \ell_e\right]$. As space collocation points $X_{e,0}$ for \eqref{eq:misfit_initial} we used $n_0 = 1000$ randomly selected points $ \left\{ x_{e,0}^i \right\}_{i=1}^{n_0} \subset \left[0, \ell_e\right]$. For \eqref{eq:misfit_Kirchhoff}, \eqref{eq:misfit_Dirichlet} and \eqref{eq:misfit_continuity} (or \eqref{eq:misfit_continuity_average}) we use $n_b = 1000$ randomly chosen points $\left\{ t_{0,b}^i \right\}_{i=1}^{n_b} \subset \left( 0,T \right)$ as time snapshots $X_{0,b}$ for origin vertices, i.e. $v = 0$, and $n_b = 1000$ randomly chosen points $\left\{ t_{1,b}^i \right\}_{i=1}^{n_b} \subset \left( 0,T \right)$ as time snapshots $X_{1,b}$ for terminal vertices, i.e. $v = 1$ (we recall that all edges of the model graph $\Gamma$ have length $1$). 

	The results of the fourth and fifth numerical experiments are shown in Table~\ref{tab:time_stepping_1} and Table~\ref{tab:time_stepping_2}, where the relative $L^2$-errors of the values generated with the discrete \graphPINN \ approach, as described in Section~\ref{subsec:discrete_graphPINN}, are listed. 
	For the results of these two numerical experiments one feed-forward neural network was used for each edge of the graph, as described by \eqref{one_for_each}. For the fourth numerical experiment, which produced the values in Table~\ref{tab:time_stepping_1}, we chose a discretization of $n_t = 200$ and for the fifth numerical experiment, which produced the values in Table~\ref{tab:time_stepping_1}, we chose a discretization of $n_t = 400$. Here too, as in the continuous \graphPINN \ approach, we used the same combination of \adam\ and \lbfgs \ optimizers.
    While the computation of the first iterate $\rho^1_{\theta_e}$ allows for up to $\num{2000}$ iterates with a learning rate of $0.01$ followed by up to $\num{10000}$ \lbfgs\ iterations, the remaining iterates only allow for up to $\num{100}$ \adam\ steps with a learning rate of $0.001$ followed by up to $\num{10000}$ \lbfgs\ iterations.
    Here, the network of the previous time-layer is used to warm-start the optimizer.
    
    Among all discussed methods, the discrete \graphPINN\ approach is the only one that achieves an accuracy comparable to the one obtained by the reference \fvm\ implementation.
    In addition, it is the only method for which the \lbfgs\ optimizer terminates due to a small error, and not by exhausting the maximum iteration numbers. Note that the results presented in Table~\ref{tab:one_net} produce relative errors larger than one in many cases and as such render the approach in its current form not viable.
    Although a direct computation of the solution on the time-space domain seems preferable, an exploitation of the parabolic structure of the PDE, as was done in the discrete time approach by the fully implicit Euler scheme, seems to improve both the accuracy and runtime.
    Finally, we mention that all of the discussed \pinn-based approaches are much slower than the reference \fvm\ algorithm.
    Here, the solution of multiple linear equations necessary to compute the solution in classical approaches like finite element methods or finite volume methods is much faster the the solution of one or more highly nonlinear optimization problems.
    Note however, that the \fvm\-method is purpose-built for this specific type of equations, while the \pinn\-based approaches might be used (with other hyperparameters, architectures, etc.) for other PDEs as well, including inverse problems.




\begin{figure}
    \begin{center}
        \begin{subfigure}[b]{0.3\textwidth}
            \begin{center}
                \includegraphics[scale=0.14]{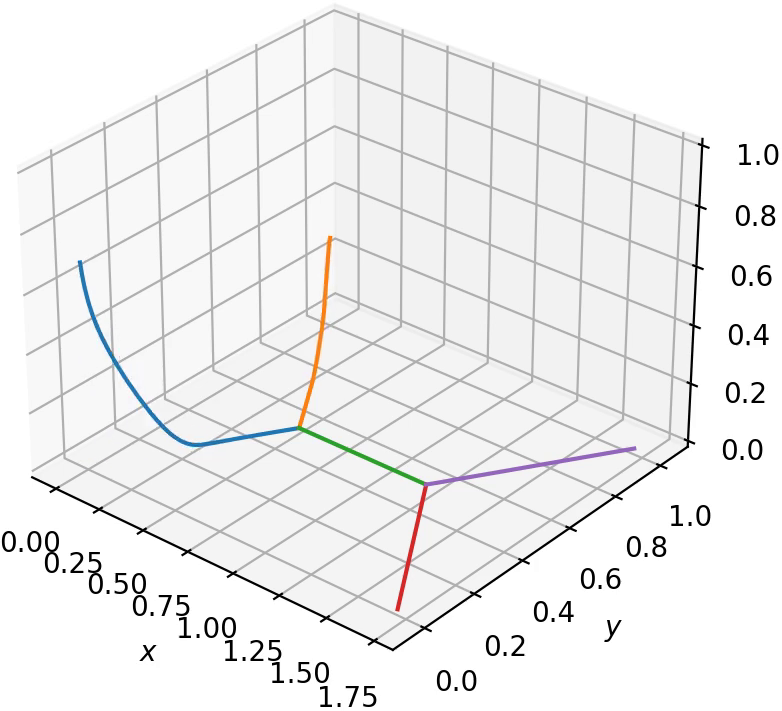}
            \end{center}
            \caption{$t=0.5$}
        \end{subfigure}
        \begin{subfigure}[b]{0.3\textwidth}
            \begin{center}
                \includegraphics[scale=0.14]{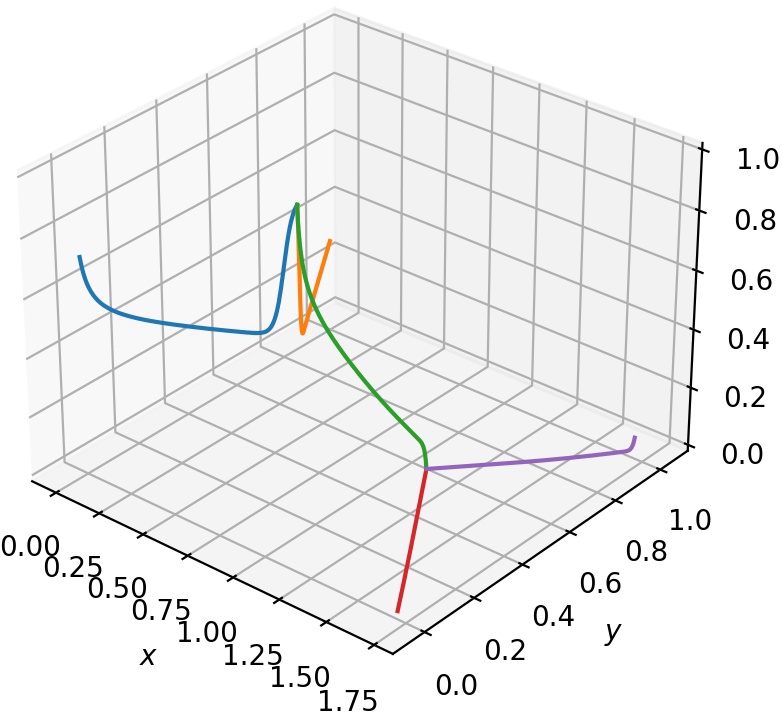}
            \end{center}
            \caption{$t=3.0$}
        \end{subfigure}
        \begin{subfigure}[b]{0.3\textwidth}
            \begin{center}
                \includegraphics[scale=0.14]{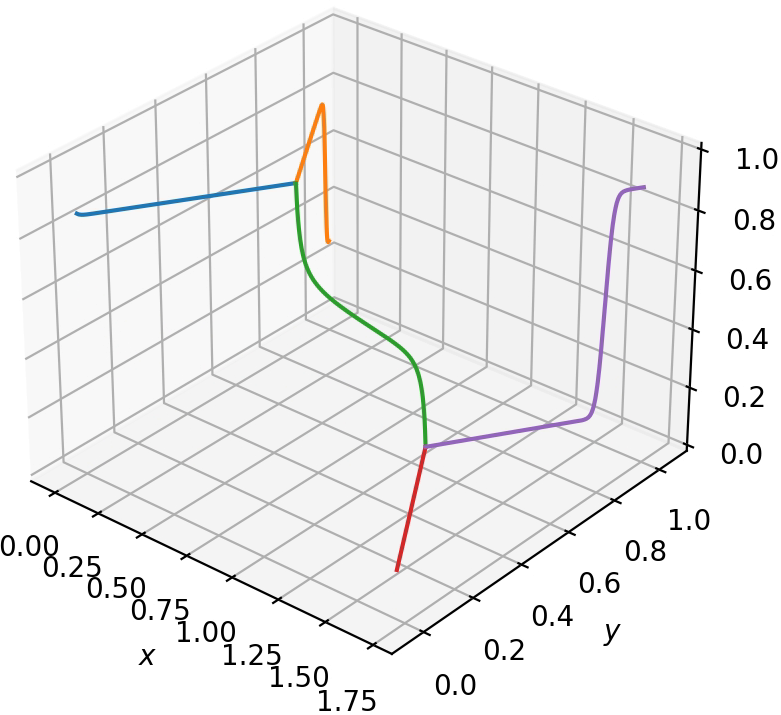}
            \end{center}
            \caption{$t=9.0$}
        \end{subfigure}
    \end{center}
    \caption{\pinn\ solution obtained with the space-time appoach after \num{1000} \adam\ steps and \num{50000} \lbfgs\ steps. Number of hidden layers is $3$, number of neurons per layer is $20$, activation function is $\tanh$ on hidden and output layers.
    A video of the \pinn\ solution can be found at \protect\url{https://vimeo.com/705483399}, the FVM reference solution can be found at \protect\url{https://vimeo.com/705483518} and the difference between both at \protect\url{https://vimeo.com/705483580}.}
    \label{fig:solution}
\end{figure}

	\begin{table}
			\begin{tabular}{llllll}
                \toprule
                neurons \textbackslash \ layers & 1  & 2 & 3 & 4 &  \\ \midrule
                10 & \num{0.51} & \num{0.24}   & \num{0.26}  & \num{0.33}  \\
                20 & \num{0.45} & \num{0.24}  & \num{0.15}  &  \num{0.42} \\
                30 & \num{0.41} & \num{0.24} &  \num{0.36} &  \num{0.32} \\
                40 & \num{0.38} & \num{0.12}  & \num{0.12}  & \num{0.20}  \\ \bottomrule
			\end{tabular}
		\caption{Relative $L^2$-error, space time approach, one NN per edge, 1000 \adam\ steps followed by up to 50000 \lbfgs\ steps.
        $N_0 = N_b = 1000$, $N_r = 4000$.}
		\label{tab:SpaceTime_each_edge}
	\end{table}

	\begin{table}
			\begin{tabular}{llllll}
                \toprule
                neurons \textbackslash \ layers & 1  & 2 & 3 & 4 &  \\ \midrule
                10 & \num{1.21} & \num{1.08}   & \num{1.13}  & \num{0.27}  \\
                20 & \num{1.18} & \num{1.14}  & \num{1.09}  &  \num{0.27} \\
                30 & \num{1.12} & \num{1.09}  & \num{1.13}  &  \num{1.13} \\
                40 & \num{1.11} & \num{1.15 }& \num{1.16} & \num{1.10} \\ \bottomrule
			\end{tabular}
		\caption{Relative $L^2$-error, space-time, One-Net approach, 1000 \adam\ steps followed by up to 50000 \lbfgs\ steps.
        $N_0 = N_b = 1000$, $N_r=4000$.}
        \label{tab:one_net}
	\end{table}

	\begin{table}
			\begin{tabular}{llllll}
                \toprule
                neurons \textbackslash \ layers & 1  & 2 & 3 & 4 &  \\ \midrule
                10 & \num{0.205} & \num{0.026}   & \num{0.026}  & \num{0.382}  \\
                20 & \num{0.080} & \num{0.026}  & \num{0.026}  &  \num{0.026} \\
                30 & \num{0.030} & \num{0.026}  & \num{0.026}  &  \num{0.026} \\
                40 & \num{0.030} & \num{0.026}  & \num{0.026}  &  \num{0.026} \\ \bottomrule
			\end{tabular}
		\caption{Relative $L^2$-error, discrete time stepping approach, one NN per edge, 1000 \adam\ steps followed by up to 10000 \lbfgs\ steps at each discrete time step.
        $N_0 = N_b = 200$ chosen equidistantly.}
		\label{tab:time_stepping_1}
	\end{table}

	\begin{table}
			\begin{tabular}{llllll}
                \toprule
                neurons \textbackslash \ layers & 1  & 2 & 3 & 4 &  \\ \midrule
                10 & \num{0.284} & \num{0.016}   & \num{0.015}  & \num{0.015}  \\
                20 & \num{0.113} & \num{0.015}  & \num{0.015}  &  \num{0.014} \\
                30 & \num{0.021} & \num{0.016}  & \num{0.015}  &  \num{0.015} \\
                40 & \num{0.033} & \num{0.015}  & \num{0.015}  &  \num{0.015} \\ \bottomrule
			\end{tabular}
		\caption{Relative $L^2$-error, discrete time stepping approach, one NN per edge, 1000 \adam\ steps followed by up to 10000 \lbfgs\ steps at each discrete time step.
        $N_0 = N_b = 400$ chosen equidistantly.}
		\label{tab:time_stepping_2}
	\end{table}


\section{Conclusion and outlook}
In this paper we compared several different \pinn\ frameworks to solve a differential equation on a metric graph domain. Such system arise in many important applications and usually numerical methods, such as finite volume or finite element schemes are applied. We illustrated that \pinns\ provide an alternative approach and we focus on various setups from continuous to discrete time schemes. In particular, the space time networks we employed showed good performance and should be combined in the future with more advanced time-stepping schemes such the family of Runge-Kutta schemes.
A performance analysis for various parameter identification tasks should become a topic of future research. Focusing on increased computational performance especially when the PDEs will be solved on large-scale networks is another task that requires further attention in the future. For this more sophisticated schemes as well as more tailored implementation for GPUs should also be investigated. 

	\begin{backmatter}

		\section*{Funding}
        Jan Blechschmidt acknowledges the support of the German Federal Ministry of Education and Research (BMBF) as part of the project SOPRANN -- Synthese optimaler Regelungen und adaptiver Neuronaler Netze für Mobilitätsanwendungen (05M20OCA). Martin Stoll is partially supported by the
        BMBF grant 01|S20053A (project SA$\ell$E). 
		
		
		\bibliographystyle{bmc-mathphys} 
		\bibliography{bib_graphPINNs}      
		

	\end{backmatter}
\end{document}